\documentclass{article}

\usepackage{efrenstyle}
\usepackage{multirow}
\usepackage{authblk}

\usepackage[utf8]{inputenc} 
\usepackage[T1]{fontenc}    
\usepackage{hyperref}       
\usepackage{url}            
\usepackage{booktabs}       
\usepackage{amsfonts}       
\usepackage{nicefrac}       

\title{Consistent Kernel Density Estimation with Non-Vanishing Bandwidth}
\author{Efr\'en~Cruz~Cort\'es\thanks{Contact first author at encc@umich.edu for further questions about this work.} }%
\author{Clayton~Scott}
\affil{Electrical Engineering and Computer Science\\ University of Michigan}
\renewcommand\footnotemark{}
\thanks{This work was supported in part by NSF grant 1422157.}

\date{}

\begin{document}

\maketitle

\begin{abstract}
Consistency of the kernel density estimator requires that the kernel bandwidth tends to zero as the sample size grows. In this paper we investigate the question of whether consistency is possible when the bandwidth is fixed, if we consider a more general class of weighted KDEs. To answer this question in the affirmative, we introduce the fixed-bandwidth KDE (fbKDE), obtained by solving a quadratic program, and prove that it consistently estimates any continuous square-integrable density. We also establish rates of convergence for the fbKDE with radial kernels and the box kernel under appropriate smoothness assumptions. Furthermore, in an experimental study we demonstrate that the fbKDE compares favorably to the standard KDE and the previously proposed variable bandwidth KDE.
\end{abstract}

\section{Introduction}\label{introduction}
Given an iid sample $X_1, \ldots, X_n \in \bbR^d$ drawn according to a probability density $f$, the kernel density estimator is
$$
f_{KDE} = \frac1{n} \sum_{i=1}^n k(\cdot, X_i),
$$
where $k$ is a kernel function with parameter $\sigma$. Examples of kernels are functions of the form $k(x,x') = \sigma^{-d}g(\sigma^{-2}\norm{x-x'}^2)$, where $\int{\sigma^{-d}g(\sigma^{-2}\norm{x-x'}^2)dx} = 1$ for all $x'$. A common kernel for density estimation is the Gaussian kernel $k(x,x') = (2\pi\sigma^2)^{-d/2}\exp({-(2\sigma^2)^{-1}\norm{x-x'}^2})$. Since its inception by Fix and Hodges~\cite{fix1951discriminatory} and development by Rosenblatt and Parzen \cite{rosenblatt1956remarks,parzen1962estimation}, the KDE has found numerous applications across a broad range of quantitative fields, and has also been the subject of extensive theoretical investigations, spawning several books (see, e.g, \cite{silverman86,scottdw92,wand94,devroye01}) and hundreds of research articles.

A strength of the KDE is that it makes few assumptions about $f$ and that it is consistent, meaning that it converges to $f$ as $n \to \infty$ \cite{wied2012consistency}. Analysis of the KDE stems from the following application of the triangle inequality in some norm of interest, where~$*$ denotes convolution: $\| f_{KDE} - f \| \le \| f_{KDE} - f * k\| + \| f * k- f \|$. Critical to the analysis of the KDE is the dependence of the {\em bandwidth} parameter $\sigma$ on $n$.
The first term tends to zero provided $n \sigma^d \to \infty$, i.e, the number of data points per unit volume tends to infinity. This is shown using properties of convolutions (since $f_{KDE}$ and $f * k$ may be viewed as convolutions of the kernel with the empirical and true distributions, respectively) and concentration of measure. For the latter term to tend to zero, $\sigma \to 0$ is necessary, so that the kernel behaves like a Dirac measure. With additional assumptions on the smoothness of $f$, the optimal growth of $\sigma$ as a function of $n$ can be determined.

The choice of bandwidth determines the performance of the KDE, and automatically selecting the optimal kernel bandwidth remains a difficult problem. Thus, researchers have developed numerous plug-in rules and cross-validation strategies, all of which are successful in some situations. A recent survey counts no fewer than 30 methods in the literature and cites 6 earlier review papers on the topic \cite{heidenreich2013bandwidth}.

As an alternative to the standard KDE, some authors have investigated weighted KDEs, which have the form
$
f_\alpha = \sum_{i=1}^n \alpha_i k(\cdot, X_i).
$
The weight vector $\alpha = (\alpha_1, \ldots, \alpha_n) \in \bbR^n$ is learned according to some criterion, and such weighted KDEs have been shown to yield improved performance over the standard KDE in certain situations \cite{kim1995thesis, girolami03,ganti2011cake,kim2012robust}. Consistency of such estimators has been investigated, but still under the assumption that $\sigma \to 0$ with $n$ \cite{kim2010l2, rigollet2007linear}.

In this work we consider the question of whether it is possible to learn the weights of a weighted KDE such that the resulting density estimator is consistent, for a broad class of $f$, where the bandwidth $\sigma$ remains fixed as $n \to \infty$. This question is of theoretical interest, given that all prior work establishing consistency of a KDE, to our knowledge, requires that the bandwidth shrinks to zero. The question is also of practical interest, since such a density estimator could potentially be less sensitive to the choice of bandwidth than the standard KDE, which, as mentioned above, is the main factor limiting the successful application of KDEs in practice.

In Section~\ref{fbkde} below, we introduce a weighted KDE that we refer to as the fixed-bandwidth KDE (fbKDE). Its connection to related work is given in Section~\ref{relatedwork}. The theoretical properties of this estimator, including consistency and rates of convergence with a fixed bandwidth, are established in Section~\ref{results}. Our analysis relies on reproducing kernel Hilbert spaces, a common structure in machine learning that has seldom been used to understand KDE's. In Section~\ref{experiments}, a simulation study is conducted to compare the fbKDE against the standard KDE and another weighted KDE from the literature. Our results indicate that the fbKDE is a promising alternative to these other methods of density estimation.

\section{Fixed bandwidth KDE}\label{fbkde}
We start by assuming access to iid data $X_1,\dots,X_n$ sampled from an unknown distribution with density $f$, having support $\sXbar = \supp \set{f}$ contained in the known set $\sX\subset\rd$, and with dominating measure $\mu$. The set $\sX$ is either compact, in which case $\mu$ is taken to be the Lebesgue measure, or $\sX = \rd$, in which case $\mu$ is a known finite measure. We study a weighted kernel density estimator of the form
\begin{align}\label{eq:falpha}
f_\alpha = \mysum \alpha_i k(\cdot,Z_i)
\end{align}
where $\alpha:=(\alpha_1,\dots,\alpha_n)^T\in A_n \subset \bbR^n$, $Z_i = X_i + \Gamma_i$, and $\Gamma_i$ is sampled iid from a known distribution with density $f_\Gamma$. Here, $f_\Gamma$ is chosen to ensure $Z_i\in\sX$, but not necessarily in $\sXbar$. Note that $f_\alpha$ is defined on $\sX$ and $k$ on $\sX\times\sX$. Throughout, $A_n$ is taken to be an $\ell_1$ ball in $\bbR^n$, that is
$
A_n = \set{\alpha\in\bbR^n \,\,\, | \,\,\, \onenorm{\alpha}\leq R_n},
$
for $R_n\in \bbR$. The reason for centering the kernels at $Z_i = Z_i + \Gamma_i$ instead of $X_i$ is that to accurately approximate $f$ with a fixed bandwidth, we might need centers outside the support of $f$.

To measure the error between $f_\alpha$ to $f$ we may consider the $\sL^2_\mu\paren{\sXbar}$ difference, where $\ltwomu(\sXbar)$ is the space of equivalence classes of square integrable functions, and where (using both $f$ for the function and its equivalence class) $\norm{f}_{\ltwomu(\sXbar)}^2:=\int_{\sXbar}f^2d\mu$. However, we do not know the set $\sXbar$ and cannot compute said difference. Hence, we consider the $\ltwomu\paren{\sX}$ difference from $f_\alpha$ to $f$.

To determine the scaling coefficients $\alpha$ we consider minimizing $\norm{f-f_\alpha}_{\ltwomu\paren{\sX}}^2$ over $\alpha\in A_n$. Since $\norm{f-f_\alpha}_{\ltwomu\paren{\sX}}^2 = \int_\sX {f_\alpha^2} -2\int_\sX ff_\alpha+\int_\sX{f^2}$, the term $\int_\sX f^2$ is independent of $\alpha$, and $f$ is zero outside $S$, we can focus on minimizing
$$
J(\alpha) = \int_{\sX}{f_\alpha(x)^2}d\mu(x) - 2\int_{\sXbar}{f_\alpha(x)f(x)d\mu(x)}.
$$
 Define $H(\alpha) := \int_{\sXbar} f_\alpha(x)f(x)d\mu(x) = \mysum \alpha_i \int_{\sXbar} k(x,Z_i)f(x)d\mu(x) =: \mysum \alpha_i h_i$. Since we don't know $f$, we don't know the true form of $H(\alpha)$ and $J(\alpha)$. However, the terms $h_i$ are expectations with respect to $f$ so we can estimate the term $H(\alpha)$ using the available data $\set{X_i}_{i=1}^n$. We use the leave-one-out estimator
$$
H_n(\alpha):=\mysum \alpha_i \frac{1}{n-1}\sum_{j\neq i} k(X_j,Z_i) =: \mysum \alpha_i \widehat{h}_i.
$$
With the aid of $H_n(\alpha)$, we define the function
\begin{align}
{J}_n(\alpha) &:= \int_{\sX} {f_\alpha(x)^2d\mu(x)} -2H_n(\alpha)
\end{align}
to which we have access. Let $J^*:=\inf_{\alpha\in A_n}J(\alpha)$ and
\begin{align}\label{eq:alphahat}
{\alphahat} := \argmin_{\alpha \in A_n}{{J}_n(\alpha)}.
\end{align}
The estimator
\begin{align}\label{eq:fbkde}
f_{\alphahat} = \mysum\alphahat_i k(\cdot,Z_i)
\end{align}
is referred to as the fixed bandwidth kernel density estimator (fbKDE), and is the subject of our study. In the following sections this estimator is shown to be consistent for a fixed kernel bandwidth $\sigma$ under certain conditions on $f,k,R_n,$ and $f_\Gamma$.

\section{Related work}\label{relatedwork}
The use of the $\sL^2$ norm as an objective function for kernel density estimation is not new, and in fact, choosing $\sigma$ to minimize $J_n$, with $\alpha_i = 1/n$, is the so-called least-squares leave-one-out cross-validation technique for bandwidth selection. Minimizing $J_n$ subject to the constraints $\alpha_i \ge 0$ and $\sum_i \alpha_i = 1$ was proposed by \cite{kim1995thesis}, and later rediscovered by \cite{girolami03}, who also proposed an efficient procedure for solving the resulting quadratic program, and compared the estimator to the KDE experimentally. This same estimator was later analyzed by \cite{kim2010l2}, who established an oracle inequality and consistency under the usual conditions for consistency of the KDE. 

Weighted KDEs have also been developed as a means to enhance the standard KDE in various ways. For example, a weighted KDE was proposed in \cite{ganti2011cake} as a form of multiple kernel learning, where, for every data point, multiple kernels of various bandwidths were assigned, and the weights optimized using $J_n$. A robust kernel density estimator was proposed in \cite{kim2012robust}, by viewing the standard KDE as a mean in a function space, and estimating the mean robustly. To improve the computational efficiency of evaluating a KDE, several authors have investigated sparse KDEs, learned by various criteria \cite{cruz2015sparse,girolami2003probability,phillips2013epsilon,drineas2005nystrom,chen2012super}.

The one-class SVM has been shown to converge to a truncated version of $f$ in the $\sL^2$ norm \cite{vert06}. If the truncation level (determined by the SVM regularization parameter) is high enough, and the density is bounded, then $f$ is consistently estimated. An ensemble of kernel density estimators is studied by \cite{rigollet2007linear}, who introduce aggregation procedures such that a weighted combination of standard KDEs of various bandwidths performs as well as the KDE whose bandwidth is chosen by an oracle.

In the above-cited work on weighted KDEs, whenever consistency is shown, it assumes a bandwidth tending to zero. Furthermore, the weights are constrained to be nonnegative. In contrast, we allow the weights on individual kernels to be negative, and this enables our theoretical analysis below. Finally, we remark that the terms ``fixed" or ``constant" bandwidth have been used previously in the literature to refer to a KDE where each data point is assigned the same bandwidth, as opposed to a ``variable bandwidth" KDE where each data point might receive a different bandwidth. We instead use ``fixed bandwidth" to mean the bandwidth remains fixed as the sample size grows.

\section{Theoretical Properties of the fbKDE}\label{results}

\paragraph{Notation}

The space $L_\nu^p(\sX)$ is the set of functions on $\sX$ for which the $p^{th}$ power of their absolute value is $\nu$-integrable over $\sX$. $\sL^p_\nu(\sX)$ is the set of equivalence classes of $L_\nu^p(\sX)$, where two functions $g_1$ and $g_2$ are equivalent if $\int_{\sX}\abs{g_1 - g_2}^pd\nu =0$. The symbol $\twonorm{\cdot}$ will denote both the Euclidean norm in $\rd$ as well as the $\sL^2_\nu$ norm; which is used will be clear from the context, as the elements of $\rd$ will be denoted by letters towards the end of the alphabet ($x,y$, and $z$). The set $C(\sX)$ denotes the space of continuous functions on $\sX$. Finally, the support of any function $g$ is denoted by $\supp\set{g}$.

We call a function $k:\sX\times\sX\to\bbR$ a {\em symmetric positive definite kernel} (SPD) \cite{scovel2010radial} if $k$ is symmetric and satisfies the property that for arbitrary $n$ and $\set{x_i}_{i=1}^{n}$ the matrix $[k(x_i,x_j)]_{i,j=1}^n$ is positive semidefinite. Every SPD kernel has an associated Hilbert space $\sH_k(\sX)$ of real-valued functions, called the reproducing kernel Hilbert space (RKHS) of $k$, which we will denote by $\sH$ when $k$ and $\sX$ are clear from the context, and for which $k(x,x') = \inpr{k(\cdot,x),k(\cdot,x')}_{\sH}$. SPD kernels also exhibit the {\em reproducing property}, which states that for any $h\in\sH$, $h(x) = \inpr{h,k(\cdot,x)}_{\sH}$. By a {\em radial SPD kernel} we mean an SPD kernel $k$ for which there is a strictly monotone function $g$ such that $k(x,x') = g(\twonorm{x-x'})$ for any $x,x'\in\sX$. Note that the Gaussian kernel is a radial SPD kernel.

If $k$ is a radial SPD kernel then $\sup_{x,x'\in\sX}{k(x,x')}\leq C_k$ for some $C_k>0$. This holds because
$k(x,x') = \inpr{k(\cdot,x),k(\cdot,x')} \leq \norm{k(\cdot,x)}\norm{k(\cdot,x')}$ by the reproducing property and Cauchy-Schwarz, and $\norm{k(\cdot,x)}\norm{k(\cdot,x')}= \sqrt{k(x,x)}\sqrt{k(x',x')} = g(0)$.

We will make use of the following assumptions:
\paragraph{A0}\label{assump:xset}
The set $\sX$ is either a compact subset of $\rd$, in which case $\mu$ is taken to be the Lebesgue measure, or $\sX = \rd$, in which case $\mu$ is a known finite measure.

This assumption is always held throughout the paper. It will not be explicitly stated in the statements of results below, but will be remembered in the proofs when needed.

\paragraph{A1}\label{assump:data}
$X_1,\dots,X_n$ are sampled independently and identically distributed (iid) according to $f$. $\Gamma_1,\dots,\Gamma_n$ are sampled iid according to $f_\Gamma$. Furthermore $X_1,\dots, X_n, \Gamma_i,\dots, \Gamma_n$ are independent.

Given $\set{(X_i,\Gamma_i)}_{i=1}^n$ as in~\nameref{assump:data}, we define $Z_i:=X_i+\Gamma_i$ for $1\leq i\leq n$. This notation will be kept throughout the paper.

\paragraph{A2} \label{assump:pdkernel}
The kernel $k$ is radial and SPD, with $k(x,x) = C_k$ for all $x\in\sX$. Furthermore, $k$ is Lipschitz, that is, there exists $L_k>0$ such that $\norm{k(\cdot,x) - k(\cdot,y)}_{2}\leq L_k\norm{x-y}_{2}$ holds for all $x$, $y\in\sX$.

Recall from eqn.~\eqref{eq:alphahat} that $\alphahat$ is the minimizer of ${J}_n$ over $A_n$. To show the consistency of $f_\alpha$ the overall approach of the following sections will be to show that $J(\alphahat)$ is close to $J^* = \inf_{\alpha\in A_n}\twonorm{f - f_\alpha}^2$ with high probability, and then show that $J(\alpha)$ (and therefore $J^*$) can be made arbitrarily small for optimal choice of $\alpha$. We start by stating an oracle inequality relating $J(\alphahat)$ and $J^*$.
\begin{lemma}
\label{lemma:oracle}
Let $\epsilon>0$. Let $\set{(X_i,\Gamma_i)}_{i=1}^n$ satisfy assumption~\nameref{assump:data}. Let $k$ satisfy $\sup_{x,x'\in\sX}{k(x,x')}\leq C_k$ and $f_\alpha$ be as in Equation~\eqref{eq:falpha}. Let $\delta = 2n\exp\paren{- \frac{(n-1)\epsilon^2}{8C_k^2R_n^2}}$. With probability $\geq 1-\delta$,
$$
\twonorm{f - f_{\alphahat}}^2 \leq \epsilon + \inf_{\alpha\in A_n}\twonorm{f - f_\alpha}^2.
$$
\hfill \qed
\end{lemma}
The proof consists of showing the terms $H_n$ of $J_n$ concentrate around the terms $H$ of $J$, and is placed in Section~\ref{proofs}. This result allows us to focus on the term $J^* = \inf_{\alpha\in A_n}\twonorm{f - f_\alpha}^2$, which we proceed to do in the following sections.

\subsection{Consistency of $f_\alphahat$}\label{gaussianconsistency}
We state the consistency of $f_{\alphahat}$ for radial SPD kernels:
\begin{thm}\label{thm:consistency}
Let $\epsilon>0$. Let $\set{(X_i,\Gamma_i)}_{i=1}^n$ satisfy~\nameref{assump:data}, where $f\in\sF(\sX):= L_{\mu}^{2}(\sX)\cap C(\sX)$ and $\supp\set{f_Z}\supseteq\sX$. Let $k$ satisfy~\nameref{assump:pdkernel} and $f_{\alphahat}$ be as in Equation~\eqref{eq:fbkde}. If $A_n$ is such that $R_n\rightarrow \infty$ but $R_n^2 {\log{n}}/{n}\rightarrow 0$ as $n\rightarrow\infty$, then
$$
\prob{X,\Gamma}{ \twonorm{f - f_{\alphahat}}^2>\epsilon}\rightarrow 0
$$
as $n\rightarrow\infty$.
\hfill \qed
\end{thm}
The probability $\prob{X,\Gamma}{}$ is the joint probability of $\set{(X_i,\Gamma_i)}_{i=1}^n$. The sketch of the proof is as follows. To analyze the term $J^* = \inf_{\alpha\in A_n}\twonorm{f - f_\alpha}^2$ from Lemma~\ref{lemma:oracle}, we use the fact that $\sH$ is dense in $\sF$ \cite{scovel2010radial} in the sup-norm sense and that $\sH^0 : = \set{\sum^N_j c_j k(\cdot,y_j) | y_j \in \sX, c_i \in \bbR, \text{and}\,\,\,N\in\mathbb{N}}$ is dense in $\sH$ in the $\sH-$norm sense \cite{steinwart08}. That is, there exists a function $f_\sH\in\sH$ arbitrarily close to $f$ and a function $f_\beta\in\sH^0$ arbitrarily close to $f_\sH$. The function $f_\beta$ has the form
\begin{align}\label{eq:fbeta}
f_\beta: = \sum_{j=1}^m{\beta_j k(\cdot,y_j)},
\end{align}
where $\beta = (\beta_1,\dots,\beta_m)^T$. Note this is an abuse of notation since the functions $f_\alpha$ and $f_\beta$ do not have the same centers nor necessarily the same number of components. By the triangle inequality we have for any $\alpha$ in $A_n$:
\begin{align}\label{eq:triangledecomposition}
\twonorm{f - f_\alpha} \leq \twonorm{f - f_\sH} + \twonorm{f_\sH - f_\beta} + \twonorm{f_\beta - f_\alpha}.
\end{align}
By the above denseness results, the first two terms are small. To make the third term small we need two things: that $R_n$ is large enough so that there is an $\alpha\in A_n$ matching $\beta$, and that there exists centers $\set{Z_i}_{i=1}^n$ of $f_\alpha$ that are close to the centers $\set{y_j}_{j=1}^m$ of $f_\beta$, which will be true with a certain probability. In Section~\ref{proofs} we will prove all these approximations and show that the relevant probability is indeed large and approaches one.

\subsection{Convergence rates of $f_\alphahat$ for SPD radial kernels}\label{ratesgaussian}
The rates for radial SPD kernels may be slow, since these kernels are "universal" in that they can approximate arbitrary functions in $L^{2}_{\mu}(\sX)\cap C(\sX)$. To get better rates, we can make stronger assumptions on $f$. Thus, let $\sF_k =\set{f \,\,\, | \,\,\, f\geq 0 \,\, a.e., \int f(x)d\mu(x) = 1, \,\,\, f = \int_\sX k (\cdot,x)\lambda(x)dx, \,\,\, \lambda\in L^1(\sX)}$, that is, the space of densities expressible as $k$-smoothed $L_1$ functions. Then we obtain the following convergence rates.
\begin{thm}\label{theorem:gaussrates}
Let $\delta\in(0,1)$. Let $\sXbar = \sX$, $k$ satisfy~\nameref{assump:pdkernel}, $\set{(X_i,\Gamma_i)}_{i=1}^n$ satisfy~\nameref{assump:data} with $f\in\sF_k$, and $\min_{z\in\sX}\set{f_Z(z)}>0$. Let $f_\alphahat$ be as in Equation~\eqref{eq:fbkde}. If $d>4$ and $R_n\sim n^{1/2-d/2}$, then with probability $\geq 1-\delta$
\begin{align*}
\twonorm{f-f_\alphahat}^2 \lesssim \paren{\frac{1}{n}}^{2/d}\log^{1/2}\paren{n/\delta}.
\end{align*}
If $d\leq4$ and $R_n \sim n^{(1-C)/2}$ for C an arbitrary constant in $(0,1)$, then with probability $\geq 1-\delta$
\begin{align*}
\twonorm{f-f_\alphahat}^2 \lesssim \paren{\frac{1}{n}}^{C/2}\log^{2/d}\paren{n/\delta}.
\end{align*}
\hfill \qed
\end{thm}
The symbol $\lesssim$ indicates the first term is bounded by a positive constant (independent of $n$ and $d$) times the second term, and $\sim$ means they grow at the same rate. Note the condition $\min_{z\in\sX}\set{f_Z(z)}>0$ is satisfied, for example, if $\sX$ is compact and $f_\Gamma$ is Gaussian. The first step in proving Theorem~\ref{theorem:gaussrates} is just a reformulation of the oracle inequality from Lemma~\ref{lemma:oracle}:
\begin{lemma}\label{lemma:gaussrates1}
Let $\delta_1\in(0,1)$. Let $\set{(X_i,\Gamma_i)}_{i=1}^n$ satisfy assumption~\nameref{assump:data}. Let $k$ satisfy assumption~\nameref{assump:pdkernel} and $f_\alpha$ be as in Equation~\eqref{eq:falpha}. Then with probability $\geq 1 - \delta_1$
\begin{align*}
J(\alphahat)\leq \epsilon_1(n) + J^*
\end{align*}
where $\epsilon_1(n): = \sqrt{8}C_k R_n\sqrt{\frac{\log{(4n/\delta_1)}}{n-1}}$.
\hfill \qed
\end{lemma}
Now, for the $J^*$ term in Lemma~\ref{lemma:gaussrates1} we introduce the function $f_\beta$ as in eqn~\eqref{eq:fbeta} and make the following decomposition, valid for any $\alpha\in A_n$:
$
\twonorm{f - f_\alpha} \leq \twonorm{f - f_\beta} + \twonorm{f_\beta - f_\alpha}.
$
The following lemma concerns the term $\twonorm{f - f_\beta}$, and is taken from \cite{gnecco2008approximation}:
\begin{lemma}\label{lemma:gaussrates2}
Let $f\in\sF_k$. For any $m\in\bbN$ there are m points $\set{y_j}_{j=1}^m\subset\sX$ and $m$ coefficients $\set{c_j}_{j=1}^m\subset\bbR$ such that
\begin{align}
\norm{f - \sum_{i=1}^m\frac{c_j\onenorm{\lambda}}{m} k(\cdot,y_j)}_\infty &\leq \epsilon_2(m) \nonumber
\end{align}
where $\epsilon_2(m): = C\onenorm{\lambda}\sqrt{\frac{V_k}{m}}$ for some absolute constant $C$ and where $V_k$ is the VC-dimension of the set $\set{k(\cdot,x) \,\,\, | \,\,\, x\in \sX}$.
\hfill \qed
\end{lemma}
\noindent The VC dimension of a set $\set{g_i}$ is the maximum number of points that can be separated arbitrarily by functions of the form $g_i-r$, $r\in\bbR$. For radial kernels, $V_k = d+1$ (see \cite{girosi1995approximation}, \cite{kon2006approximating}).

Now let $\beta_j = \frac{c_j\onenorm{\lambda}}{m}$ and $f_\beta = \sum_{j=1}^m\beta_j k(\cdot,y_j)$. For the remaining term $\twonorm{f_\beta - f_\alpha}$ we will proceed as in the proof of Theorem~\ref{thm:consistency}. That is, we need that for all $y_j$ there is a point $Z_{i_j}$ close to it, and then we can approximate $f_\beta$ with $f_\alpha = \mysum{\alpha_i k(\cdot,Z_i)}$.
\begin{lemma}\label{lemma:gaussrates3}
Let $\delta_2\in(0,1)$, let $f, m$, and $f_\beta$ be as above. Let $\set{(X_i,\Gamma_i)}_{i=1}^n$ satisfy assumption~\nameref{assump:data}. With probability $\geq 1 - \delta_2$
\begin{align*}
\inf_{\alpha\in A_n}\twonorm{f_\beta - f_\alpha}\leq \epsilon_3(n,m)
\end{align*}
where $\epsilon_3(n,m): = \frac{C}{n^{1/d}}\log^{1/d}{(m/\delta_2)}$, for $C$ a constant independent of $n$ and $m$.
\hfill \qed
\end{lemma}
Putting Lemmas~\ref{lemma:gaussrates1}, \ref{lemma:gaussrates2}, and \ref{lemma:gaussrates3} together and choosing $m=n^{\theta}$ for an appropriate $\theta$ (the exact form is shown in Section~\ref{proofs}) we obtain the proof of Theorem~\ref{theorem:gaussrates}.

\subsection{Convergence rates for the box kernel}\label{boxrates}
While the previous theorem considered radial SPD kernels, the oracle inequality applies more generally, and in this section we present rates for a nonradial kernel, the box kernel. We assume $\sX=[0,1]^d$ and that the kernel centers are predetermined according to a uniform grid. Precise details are given in the proof of Theorem~\ref{theorem:approx_rates}. Thus the fbKDE centered at $\set{y_i}_{i=1}^{M}$ is
$
\tilde{f}_\alpha:=\sum_{i=1}^{M}\alpha_i k(\cdot,y_i),
$
where the $\alpha$ weights are learned in the same way as before, the only change being the kernel centers. Let $\sF = L_1(\sX)\cap L_2(\sX)\cap \text{Lip}(\sX)$, where $\text{Lip}(\sX)$ are the Lipschitz functions on $\sX$, and let $L_f$ be the Lipschitz constant of $f$. Also, let $k$ be the box kernel $ k(x,y) = \frac{1}{(2\sigma)^{d}} \ind{\norm{x-y}_\infty\leq\sigma}$ defined on $\sX \pm \sigma \times \sX \pm \sigma$, and for simplicity assume $\sigma = \frac{1}{2q}$ for $q$ a positive integer. The following theorem is proved in Section~\ref{proofs}.

\begin{thm}\label{theorem:approx_rates}
Let $f \in \mathcal{F}$, $\supp\set{f_Z}\supset\sX$, $\set{(X_i,\Gamma_i)}_{i=1}^n$ satisfy~\nameref{assump:data}, and $R_n \sim n^{(d-1)/(2d+2)}$. Let $\delta \in \paren{0,1}$. With probability at least $1-\delta$
\begin{align*}
\twonorm{f-\tilde{f}_{\alphahat}}^2 \lesssim \frac{\log^{1/2}(n/\delta)}{n^{1/(d+1)}}.
\end{align*}
\hfill \qed
\end{thm}
As with the previous results, the stochastic error is controlled via the oracle inequality. Whereas the preceding results leveraged known approximation properties of radial SPD kernels, in the case of the box kernel we give a novel argument for approximating Lipschitz functions with box kernels having a non-vanishing bandwidth.
\section{Experimental Results}\label{experiments}
We now show the performance of the fbKDE as compared to the KDE and variable bandwidth KDE. The variable bandwidth KDE (\cite{comaniciu2001variable}), which we refer as the vKDE, has the form $f_{vKDE} = \frac{1}{n}\sum_{i=1}^n k_{\sigma_i}(x,X_i) $ where each data point has an individual bandwidth $\sigma_i$. \cite{comaniciu2001variable} proposes to use $\sigma_i = \sigma\lambda^{1/2}\paren{f_{KDE}(X_i)}^{-1/2}$ where $\sigma$ is the standard KDE's bandwidth parameter and $\lambda$ is the geometric mean of $\set{f_{KDE}(X_i)}_{i=1}^n$.

When implementing the fbKDE, there are a few considerations. First, when computing $\int_\sX f_\alpha^2(x)d\mu(x)$, the first term of $J_n(\alpha)$, we must compute the integral $\int_\sX k(x,z_i)k(x,z_j)d\mu(x)$. For computational considerations we assume $\sX=\rd$ and $\mu$ is the Lebesgue measure. This deviates from our theory, which requires finite $\mu$ for $\sX=\rd$. Thus, for the Gaussian kernel, which we use in our experiments, this leads to $\int_\sX k(x,z_i)k(x,z_j)d\mu(x) = k_{\sqrt{2}\sigma}(z_i,z_j)$. To obtain the $\alpha$ weights we have to solve a quadratic program. We used the ADMM algorithm described in~\cite{boyd2011distributed}, with the aid of the $\ell_1$ projection algorithm from~\cite{duchi2008efficient}.

\begin{figure}[t]
\begin{minipage}[t]{0.49\textwidth}
\includegraphics[width=1\textwidth]{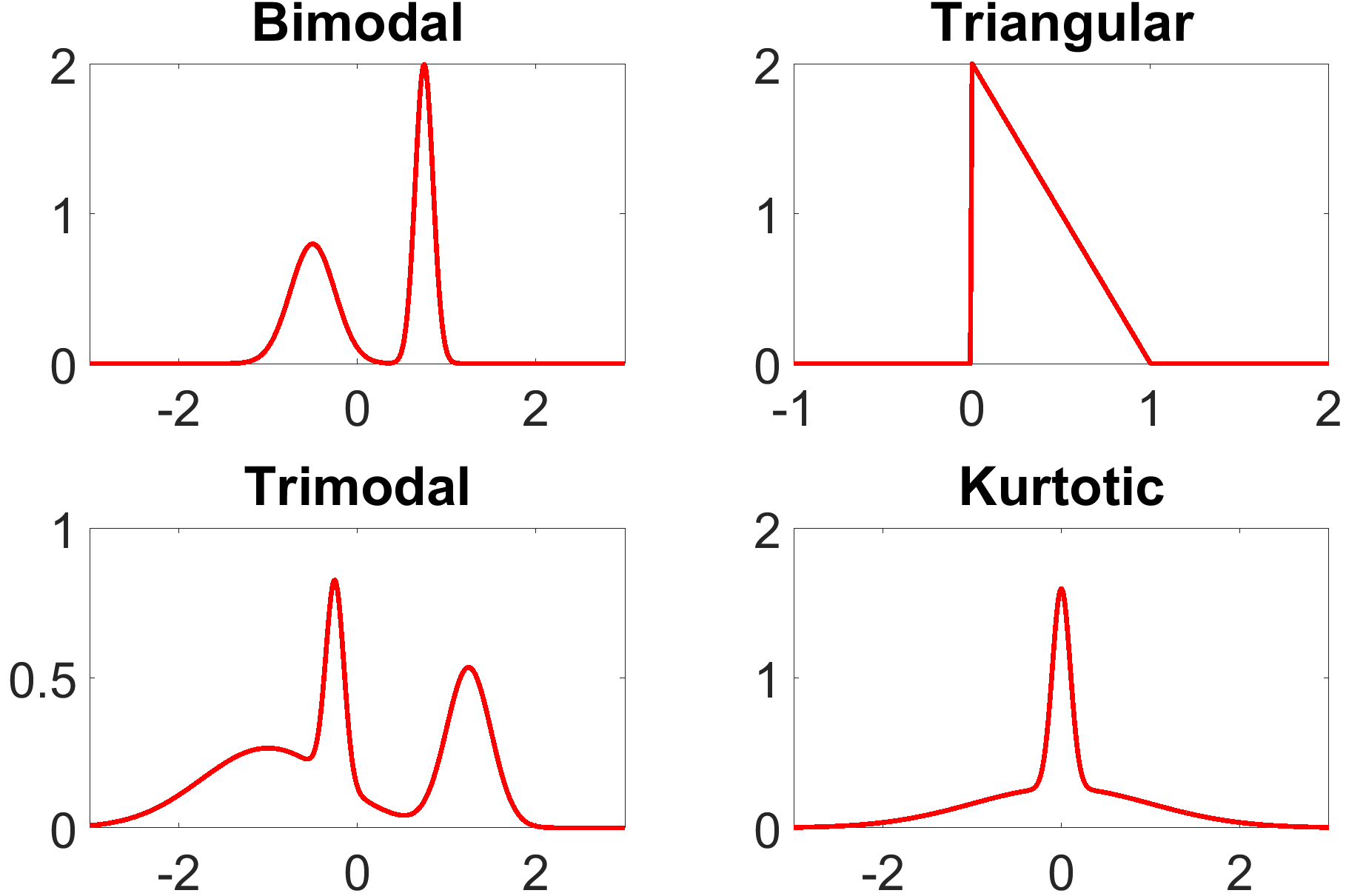}
\caption{Bimodal, triangular, trimodal and kurtotic densities used to evaluate the fbKDE performance.}
\label{fig:densities}
\end{minipage}
\hspace*{\fill} 
\begin{minipage}[t]{0.49\textwidth}
\includegraphics[width=1\textwidth]{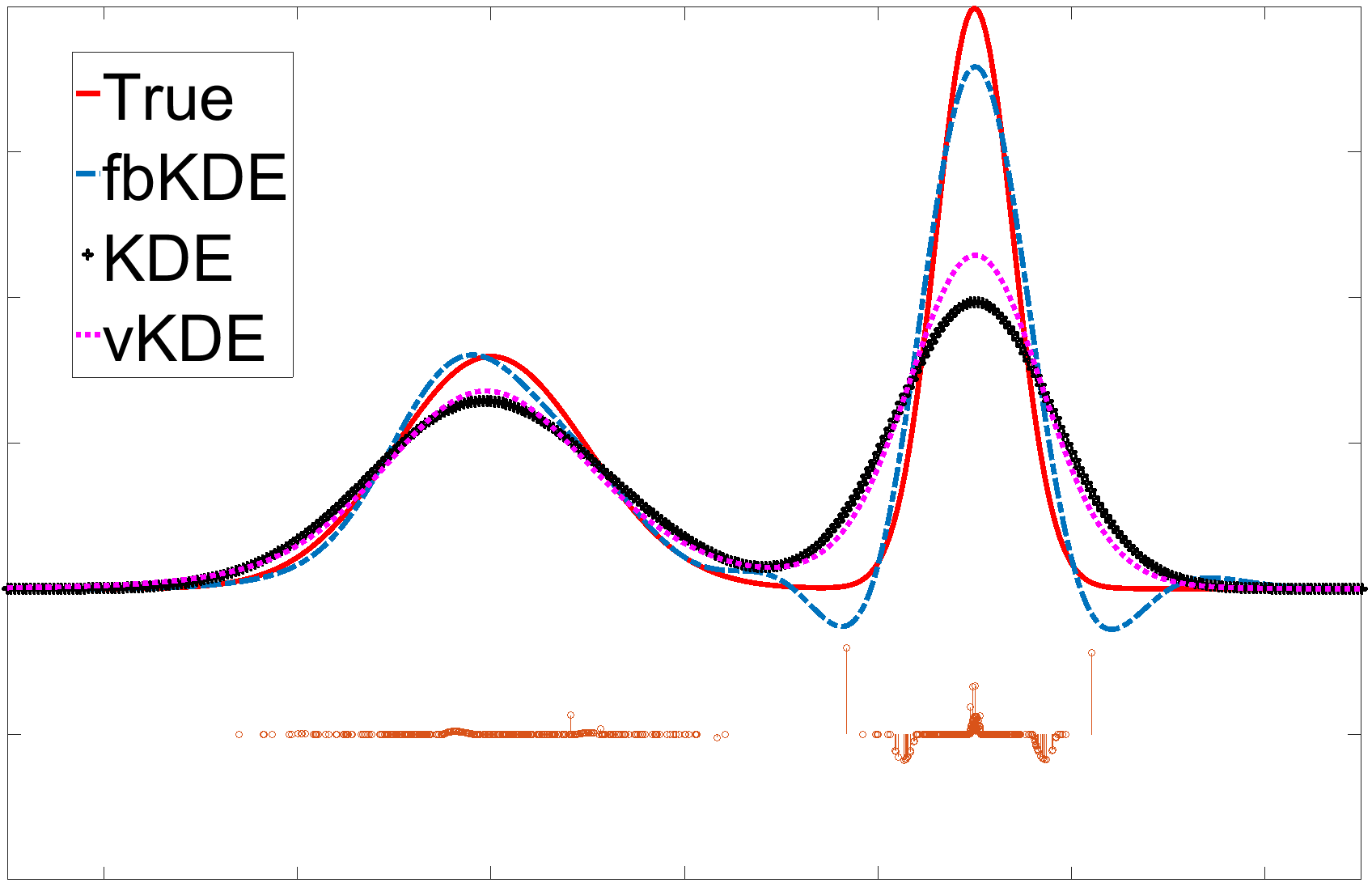}
\caption{Bimodal density and kernel estimators with training size $800$. The stem subplot indicates the values of $\alpha$ (centered offset for visualization). Note that some of the $\alpha$ weights are negative.}
\label{fig:bimodal}
\end{minipage}
\end{figure}

We examine a few benchmark real datasets as well as synthetic data from 1-dimensional densities. The synthetic densities are the triangular density as well as three Gaussian mixtures, a bimodal, a trimodal and a kurtotic, as shown in Figure~\ref{fig:densities}. We computed the parameters $\sigma$, $R_n$, and $\sigma_\gamma$ in two different ways. First we used rules of thumb. For $\sigma$, we used Silverman's rule of thumb (\cite{silverman86}). For $R_n$ we used, based on the convergence rates, $(n/\log(n))^{1/3}$ for $d\leq4$ and $(n/\log(n))^{(1/2-2/d)}$ for $d>4$. For $\sigma_\gamma$ we used, inspired by the Jaakkola heuristic~\cite{jaakkola1999using}, the median distance to the $5^{th}$ nearest neighbor. Second we used a $V$-fold CV procedure over $100$ parameter combinations drawn randomly from $\Theta_d$, with $V=2$ for $n>1000$ and $V=3$ otherwise (see~\cite{arlot2016choice}). $\Theta_d:=[.1,.5]_{\ell}\times[1.1,2\sqrt{(n)}]_{\ell}\times[.001,.1]_{\ell}$ for $d\leq 4$, and $\Theta_d:=[.1,1]_{\ell}\times[1.1,2n^{(1/2-2/d)}]_{\ell}\times[.001,.1]_{\ell}$ for $d>4$, where the subscript $\ell$ indicates logarithmic spacing and where the range is chosen thus since the data is standardized and, for the $R_n$ range, informed by the convergence rates. Finally, we used $n$ to be $4/5$ of the original data size for training and the rest for testing. We compute the value $J^T_n$ as $\int{f_\alphahat}^2 - 2\sum_{i=1}^n\alphahat_i\frac{1}{n_T}\sum_{\ell=1}^{n_T}k(x_\ell^{(T)},z_i)$, where $\set{x_\ell^{(T)}}_{\ell=1}^{n_T}$ is the test set. Figure~\ref{fig:bimodal} shows the bimodal density, the fbKDE, KDE, and vKDE along with the $\alpha$ values for the fbKDE. Table~\ref{tab:grand-comparison} shows the $J_n(\alphahat)^T$ error as well as the $\norm{\cdot}_{\infty}$ error, where for some function $g$, $\norm{g}_\infty:=\max_{x\in\sX}\abs{g(x)}$.

In Figure~\ref{fig:bimodal} the density has two Gaussian elements with different widths. It is difficult for KDE and even for the vKDE to approximate such a density. The fbKDE, however, is able to approximate components of different smoothness by making some weights negative. These weights subtract excess mass in regions of refinement. Note that by doing so the fbKDE may itself overshoot and become negative. A similar behavior is exhibited for other densities, in which smoothness varies across regions. In Table~\ref{tab:grand-comparison} we report the performance of the three estimators for both CV and rule of thumb. Note the fbKDE often performs better in terms of the $\norm{\cdot}_\infty$, and when the bandwidth is chosen according to a rule of thumb. The fbKDE outperforms both the KDE and vKDE in about half of the cases.

Finally we show, for the bimodal density, a comparison of the performance as the sample size grows. We have chosen the parameters according to the rules of thumb discussed above. Table~\ref{tab:sample_size} presents the errors. Note that as the sample size grows the KDE and vKDE do not significantly improve, even though the bandwidth is being updated according to Silverman's rule. The fbKDE leverages new observations and refines its approximation, and this effect is more obvious for the $\norm{\cdot}_\infty$ case. Indeed, the $\norm{\cdot}_\infty$ error for fbKDE is smaller at $n=50$ than at for KDE and vKDE at $n=2050$. Similar results hold for the other synthetic datasets. This highlights a notable property of the fbKDE, that it can handle densities with differing degrees of smoothness.

\begin{table}[t]
  \centering
  \caption{Performance comparison for different datasets and bandwidth selection methods. For the synthetic datasets we drew 1000 samples, $n=800$ of which were used for training.}
\scalebox{1}{
    \begin{tabular}{ll rrr rrr}
    \toprule
	& & \multicolumn{3}{c}{Rule of Thumb} &\multicolumn{3}{c}{Cross-Validation}\\
    \midrule
	& & fbKDE & KDE & vKDE & fbKDE & KDE & vKDE \\
	\midrule
\multirow{2}{*}{Bimodal} & $J_n^T$ & {\bf -1.0180}  & -0.8110 & -0.8765 & {\bf-1.0660}  & -0.9785 &  -1.0413\\
							& $\norm{\cdot}_\infty$ & {\bf0.2287}  &1.0141&  0.8468 &{\bf0.1865}   & 0.5534 &   0.2782\\
\multirow{2}{*}{Triangular} & $J_n^T$ &-1.2889  &-1.2897  &{\bf-1.2958}&{\bf-1.2332}  & -1.2095  & -1.2073\\
							& $\norm{\cdot}_\infty$ &{\bf1.0121} &1.0200 &1.0923&1.1599  &  {\bf1.1437}   & 1.2242\\
\multirow{2}{*}{Trimodal} & $J_n^T$ &{\bf-0.3317} &-0.2919 &-0.3025&-0.3457  & {\bf-0.3379}  & -0.3456\\
							& $\norm{\cdot}_\infty$ &{\bf0.2335}  &0.4571 &0.4156&0.1212  &  0.1879  &  {\bf0.1032}\\
\multirow{2}{*}{Kurtotic} & $J_n^T$ &{\bf-0.5444} &-0.4735 &-0.5271&-0.5831  & -0.5647  & {\bf-0.5894}\\
							& $\norm{\cdot}_\infty$ &{\bf0.2800} &0.8122 &0.5540&0.1379  &  0.3902  &  {\bf0.1181}\\
Banana & $J_n^T$ & -0.0838 & -0.0821 & {\bf-0.0839}  & -0.0821 & -0.0837 & {\bf-0.0853} \\
Ringnorm & $J_n^T$  & 2.4E-09 & -2.3E-10 & {\bf-2.7E-10}  & -1.7E-10 & -3.2E-10 & {\bf-3.5E-10} \\
Thyroid & $J_n^T$  & -0.0932 & -0.0448 & {\bf-0.1415}  &    {\bf-0.2765} & -0.2514 & -0.2083 \\
Diabetes & $J_n^T$  & -1.4E-05 & -0.0004 & {\bf-9.8E-04}  &    {\bf-0.0010} & -0.0007 & {\bf-0.0010} \\
Waveform & $J_n^T$  & 1.5E-09 & {\bf-1.2E-11} & 1.25E-11  & -2.1E-12 & -1.2E-11 & {\bf-1.25E-11} \\
Iris & $J_n^T$  & 0.0166 & {\bf-0.0204} & 0.0058  &    {\bf-0.0102} & 0.0027 & 0.0777 \\
    \bottomrule
    \end{tabular}%
}
  \label{tab:grand-comparison}%
\end{table}%

\begin{table}[h]
\centering
\caption{Performance comparison with respect to the sample size for the bimodal density.}
\scalebox{1}{
\begin{tabular}{ll rrr rrr r}
\toprule
&& \multicolumn{7}{c}{Sample size $n$}\\
\midrule
&& 50&250&450&1050&1650&1850&2050\\
\midrule
\multirow{3}{*}{$\norm{\cdot}_\infty$ error}
&fbKDE &{\bf0.7046}&{\bf0.5847}&{\bf0.4836}&{\bf0.3549} &{\bf0.1807}  &{\bf  0.1642}   & {\bf0.1761}\\
&KDE    &1.1341&1.0567&1.1451&1.0833 & 0.9684  &  0.9420    &0.9160\\
&vKDE   &1.0287&0.8811&1.0459&0.9670 &0.8106  &  0.7562    &0.7300\\
\midrule
\multirow{3}{*}{$J_n^T$ error} & fbKDE&{\bf-0.8985}  &{\bf -0.9623}& {\bf  -0.7487} &{\bf-1.0788} &{\bf-0.9787}  & {\bf-1.0493} & {\bf -0.9722}\\
&  KDE &-0.7099 &  -0.7639 &  -0.6859 &-0.8220&-0.8284   &-0.8728   &-0.8277\\
&  vKDE &-0.7763  &-0.8284   &-0.7091  &-0.8793   &-0.8782  & -0.9372   &-0.8839\\
\bottomrule
\end{tabular}%
}
\label{tab:sample_size}
\end{table}


\section{Conclusion}

We have presented a new member of the family of kernel estimators, the {\em fixed bandwidth kernel density estimator}, with desirable statistical properties. In particular, we showed the fbKDE is consistent for fixed kernel parameter $\sigma$, and we provided convergence rates. The fbKDE is a good alternative to the KDE in cases where computing an optimal bandwidth is difficult and for densities that are hard to approximate with inappropriate bandwidths. In these cases and as is shown in the experimental section, the fbKDE can greatly improve on the KDE. The way in which fbKDE achieves a more refined approximation is by balancing properly placed positive and negative weights, sometimes outside of the original support, which is facilitated by the $\Gamma$ variables, and which is not possible with the standard KDE. A few problems of interest remain open. We have illustrated two possible rate of convergence results, but expect many other such results are possible, depending on the choice of kernel and smoothness assumptions. It also remains an open problem to extend our results to more general domains $\sX$ and to dependent data.

\section{Proofs}\label{proofs}
\subsection{Oracle Inequality}
First recall the oracle inequality lemma:

\begin{lemma*}[Lemma~\ref{lemma:oracle}]
Let $\epsilon>0$. Let $\set{(X_i,\Gamma_i)}_{i=1}^n$ satisfy assumption~\nameref{assump:data}. Let $k$ satisfy assumption~\nameref{assump:pdkernel} and $f_\alpha$ be as in Equation~\eqref{eq:falpha}. Let $\delta = 2n\exp\paren{- \frac{(n-1)\epsilon^2}{8C_k^2R_n^2}}$. With probability $\geq 1-\delta$,
\begin{align*}
\twonorm{f - f_{\alphahat}}^2 \leq \epsilon + \inf_{\alpha\in A_n}\twonorm{f - f_\alpha}^2.
\end{align*}
\end{lemma*}

\begin{proof}[Proof of Lemma \ref{lemma:oracle}]
Recalling the following definition from Section~\ref{fbkde}, we have
\begin{align*}
h_i = h_i(X_i,\Gamma_i) := \int_{\sXbar}k(x,X_i+\Gamma_i)f(x)d\mu(x),
\end{align*}
and
\begin{align*}
\widehat{h}_i = \widehat{h}_i (X,\Gamma_i): = \frac{1}{n-1}\sum_{j\neq i}k(X_j,X_i+\Gamma_i),
\end{align*}
were we have used the simplified notation $(X,\Gamma_i)$ to represent $(X_1,\dots,X_n,\Gamma_i)$. To simplify notation further, we  let $X_{/i}$ represent $(X_1,\dots,X_{i-1},X_{i+1},\dots,X_n)$ and use $\prob{X,\Gamma}{\cdot}$ for $\prob{X_1,\dots,X_n,\Gamma_1,\dots,\Gamma_n}{\cdot}$, and the same goes for $\expec{X,\Gamma}{\cdot}$. We now look at the probability that $H_n(\alpha)$ is close to $H(\alpha)$. We have
\begin{align}
\prob{X,\Gamma}{\sup_{\alpha \in A_n} \abs{H_n(\alpha) - H(\alpha)}>\epsilon}& = \prob{X,\Gamma}{\sup_{\alpha \in A_n} \abs{\mysum\alpha_i\hih(X,\Gamma_i) - \mysum\alpha_i h_i(X_i,\Gamma_i)}>\epsilon} \nonumber \\
& = \prob{X,\Gamma}{\sup_{\alpha \in A_n} \abs{\mysum \alpha_i (\hih(X,\Gamma_i) - h_i(X_i,\Gamma_i))}>\epsilon} \nonumber \\
& \leq \prob{X,\Gamma}{\sup_{\alpha \in A_n} \mysum \abs{\alpha_i} \abs{\hih(X,\Gamma_i) - h_i(X_i,\Gamma_i)}>\epsilon} \nonumber \\
& \leq \prob{X,\Gamma}{R_n \max_{1\leq i\leq n} \abs{\hih(X,\Gamma_i) - h_i(X_i,\Gamma_i)}>\epsilon} \nonumber \\
& \leq \mysum \prob{X,\Gamma}{\abs{\hih(X,\Gamma_i) - h_i(X_i,\Gamma_i)}>\frac{\epsilon}{R_n}}. \nonumber
\end{align}
\begin{sloppypar}
Now let $\sA_i:=\set{(x_1,\dots,x_n,\gamma_i)\in \sXbar^n\times \supp\set{f_\Gamma} \mid \abs{\hih(x,\gamma_i)-h_i(x_i,\gamma_i)}>\frac{\epsilon}{R_n}}$ and note that $\prob{X,\Gamma}{\sA_i}=\prob{X,\Gamma_i}{\sA_i}$. We have
\begin{align*}
\prob{X,\Gamma_i}{(X_1,\dots,X_n,\Gamma_i)\in\sA_i}
& = \int_{\supp\set{f_\Gamma}}{\prob{X\mid\Gamma_i}{(X_1,\dots,X_n,\gamma_i)\in\sA_i \mid \Gamma_i = \gamma_i} f_\Gamma(\gamma_i)d\gamma_i}\\
& = \int_{\supp\set{f_\Gamma}}{\prob{X}{(X_1,\dots,X_n,\gamma_i)\in\sA_i} f_\Gamma(\gamma_i)d\gamma_i},
\end{align*}
by the independence assumption of~\nameref{assump:data}. Furthermore,
\begin{align*}
\prob{X}{(X_1,\dots,X_n,\gamma_i)\in\sA_i} 
&=\int_{\sXbar}{\prob{X_{/i}\mid X_i}{(X_1,\dots,x_i,\dots,X_n,\gamma_i)\in\sA_i \mid X_i=x_i} f(x_i)d\mu(x_i)}\\
&=\int_{\sXbar}{\prob{X_{/i}}{(X_1,\dots,x_i,\dots,X_n,\gamma_i)\in\sA_i} f(x_i)d\mu(x_i)}.
\end{align*}
We now bound the term inside the integral. Since this quantity only depends on $x_i$, $\gamma_i$ through $x_i+\gamma_i$, we abbreviate it as $\prob{X_{/i}}{(X_{/i},z_i)\in\sA_i}$, where $z_i=x_i+\gamma_i$. First, note that $h_i(x_i,\gamma_i) = \expec{X_j}{k(X_j,z_i)}$ for any $j\neq i$. Also, by assumption there is a $C_k$ such that $k(x,x') \leq C_k$ for all $x,x'$. Hence,
\begin{align}
\prob{X_{/i}}{(X_{/i},z_i)\in\sA_i} &= \prob{X_{/i}}{\abs{\frac{1}{n-1}\sum_{\substack{j=1 \\ j\neq i}}^n k(X_j,z_i) - \expec{X_{/i}}{\frac{1}{n-1}\sum_{\substack{j=1 \\ j\neq i}}^n k(X_j,z_i)}} > \epsilon} \nonumber \\
&= \prob{X_{/i}}{\abs{\frac{1}{n-1}\sum_{\substack{j=1 \\ j\neq i}}^n k(X_j,z_i) - \frac{1}{n-1}\expec{X_{/i}}{\sum_{\substack{j=1 \\ j\neq i}}^n k(X_j,z_i)} }> \epsilon} \nonumber \\
&= \prob{X_{/i}}{\abs{\sum_{\substack{j=1 \\ j\neq i}}^n k(X_j,z_i) - \expec{X_{/i}}{\sum_{\substack{j=1 \\ j\neq i}}^n k(X_j,z_i)}} > (n-1)\epsilon} \nonumber \\
&\leq 2\expon{- \frac{2(n-1)\epsilon^2}{C_k^2}} , \nonumber
\end{align}
where we have used Hoeffding's inequality. So we obtain
\begin{align*}
\prob{X,\Gamma}{\abs{\hih - h_i}>\frac{\epsilon}{R_n}} &= \int_{\supp\set{f_\Gamma}} \int_{\sXbar}\prob{X_{/i}}{(X_{/i},z_i)\in\sA_i}f(x_i)f_{\Gamma}(\gamma_i)d\mu(x_i)d\gamma_i \\
&\leq 2\expon{- \frac{2(n-1)\epsilon^2}{C_k^2R_n^2}} \cdot \int_{\supp\set{f_\Gamma}}f_{\Gamma}(\gamma_i)d\gamma_i \cdot \int_{\sXbar}f(x_i)d\mu(x_i) \\
& = 2\expon{- \frac{2(n-1)\epsilon^2}{C_k^2R_n^2}},
\end{align*}
and
\begin{align*}
\prob{X,\Gamma}{\sup_{\alpha \in A_n} \abs{H_n(\alpha) - H(\alpha)}>\epsilon} & \leq 2n\expon{- \frac{2(n-1)\epsilon^2}{C_k^2R_n^2}}.
\end{align*}
\end{sloppypar}
Therefore, letting $\delta = 2n\expon{- \frac{2(n-1)\epsilon^2}{C_k^2R_n^2}}$, for any $\alpha \in A_n$
\begin{align}
\prob{X,\Gamma}{\abs{{J}_n(\alpha) - J(\alpha)} \leq 2\epsilon} & = \prob{X,\Gamma}{\abs{H_n(\alpha) - H(\alpha)} \leq \epsilon} \nonumber\\
& = 1 - \prob{X,\Gamma}{\abs{H_n(\alpha) - H(\alpha)} > \epsilon} \nonumber\\
& \geq 1-\delta. \nonumber
\end{align}

So, with probability $\geq 1-\delta$, ${J}_n(\alpha) \leq J(\alpha) + 2\epsilon$, and $J(\alpha) \leq {J}_n(\alpha) + 2\epsilon$. Recall that ${J}_n({\alphahat}) \leq {J}_n(\alpha)$ for all $\alpha\in A_n$. Then with probability $\geq 1-\delta$,
$$
J({\alphahat}) \leq \inf_{\alpha\in A_n} J(\alpha) + 4\epsilon.
$$
If we substitute $\epsilon'=4\epsilon$ we obtain the desired result.
\end{proof}

\subsection{Consistency of $f_\alphahat$}
In the following we will make use of the fact that, for continuous positive definite radial kernels, the RKHS norm dominates the $\sup$-norm which in turn dominates the $\ltwomu(\sX)$ norm. We state this as a lemma.

\begin{lemma}\label{lemma:norms}
Let $k$, $\sX$, $\mu$ satisfy assumptions~\nameref{assump:xset} and~\nameref{assump:pdkernel}. Then for any $h\in\sH$ we have
\begin{align*}
\twonorm{h}\leq\mu^{1/2}(\sX)\norm{h}_{\infty}\leq\mu^{1/2}(\sX)C_k^{1/2}\norm{h}_{\sH}.
\end{align*}
\end{lemma}

\begin{proof}
By~\nameref{assump:xset} and~\nameref{assump:pdkernel} $k$ is bounded and continuous, and by Lemma~4.28 of \cite{steinwart08}, so is every element of $\sH$. Hence, for $h\in\sH$ and for $\sX$ either compact or all of $\rd$ the essential supremum equals the supremum, so we obtain
\begin{align*}
\twonorm{h} &=\paren{\int_{\sX}{\abs{h(x)}^2d\mu(x)}}^{1/2} \\
&\leq \mu^{1/2}(\sX) \sup_{x\in\sX}\set{\abs{h(x)}} \\
& = \mu^{1/2}(\sX)\norm{h}_\infty \\
& \leq \mu^{1/2}(\sX) \sup_{x\in\sX}\set{\abs{\inpr{h,k(\cdot,x)}_{\sH}}} \\
& \leq \mu^{1/2}(\sX) C_k^{1/2} \norm{h}_\sH
\end{align*}
where the penultimate inequality follows from the reproducing property and the last inequality is just Cauchy-Schwarz.
\end{proof}

Now to prove Theorem~\ref{thm:consistency} we need a couple intermediate lemmas. First, recall
\begin{align*}
f_\alpha := \mysum \alpha_i k(\cdot,Z_i),
\end{align*}
and
\begin{align*}
\sH^0 : = \set{\sum^N_j c_j k(\cdot,y_j) | y_j \in \sX, c_i \in \bbR, \text{and}\,\,\,N\in\mathbb{N}}.
\end{align*}
Note that $\sH^0$ is dense in $\sH$ (\cite{steinwart08}). Assume $\sH$ is dense in $\ltwomu(\sX)$. Then, given $\epsilon>0$ and $\alpha\in A_n$, there is an $f_\sH\in\sH$ and $f_\beta\in\sH^0$ of the form
\begin{align*}
f_\beta: = \sum_{j=1}^m{\beta_j k(\cdot,y_j)}
\end{align*}
such that $\twonorm{f-f_\sH}<\epsilon$ and $\norm{f_\sH-f_\beta}_{\sH}<\epsilon$. We use these functions to bound $\twonorm{f-f_\alpha}$:
\begin{align*}
\twonorm{f_\alpha - f} \leq \twonorm{f_\alpha - f_\beta} + \twonorm{f_\beta - f_\sH} + \twonorm{f_\sH - f},
\end{align*}
and show the last two terms are small in the following lemma.
\begin{lemma}\label{lemma:decomposition}
Let $\epsilon>0$. Let $k$ satisfy assumption~\nameref{assump:pdkernel} and $f\in L_{\mu}^{2}(\sX)\cap C(\sX)$. Then
\begin{align}
\twonorm{f_\alpha - f} \leq \twonorm{f_\alpha - f_\beta} + 2\mu(\sX)^{1/2}\epsilon . \nonumber
\end{align}
\end{lemma}

\begin{proof}
If $\sX$ is compact, then $\sH$ is dense in $C(\sX)$ (see \cite{scovel2010radial}). Therefore, for fixed $\epsilon$, there is an $f_\sH \in \sH$ such that
\begin{align*}
\norm{f_\sH-f}_{\infty} &\leq\epsilon, \\
\end{align*}
and by Lemma~\ref{lemma:norms} 
\begin{align*}
\twonorm{f_\sH-f} &\leq\mu(\sX)^{1/2}\epsilon. \\
\end{align*}
If $\sX=\rd$, \cite{scovel2010radial} tells us $\sH$ is dense in $\ltwomu(\sX)$, so it directly follows that, for any $\epsilon>0$, there is an $f_\sH\in\sH$ satisfying
\begin{align*}
\twonorm{f_\sH-f} &\leq\mu(\sX)^{1/2}\epsilon. \\
\end{align*}

Similarly, since $\sH^0$ is dense in $\sH$ \cite{steinwart08}, for any fixed $\epsilon$ there is an $f_\beta \in \sH^0$ such that
\begin{align*}
\norm{f_\beta - f_\sH}_{\sH}&\leq C^{-1/2}\epsilon,
\end{align*}
hence, by lemma~\ref{lemma:norms}
\begin{align*}
\twonorm{f_\beta - f_\sH} &\leq \mu(\sX)^{1/2}\epsilon.
\end{align*}
Therefore:
\begin{align}
\twonorm{f_\alpha - f} & \leq \twonorm{f_\alpha - f_\beta} + \twonorm{f_\beta - f_\sH} + \twonorm{f_\sH - f} \nonumber\\
& \leq \twonorm{f_\alpha - f_\beta} + 2\mu(\sX)^{1/2}\epsilon . \nonumber
\end{align}
\end{proof}

Note that $f_\beta\in\sH^0$ implies $f_\beta  = \sum_{j=1}^m \beta_j k(\cdot,y_j)$ for some $m\in\bbN$ and where $(\beta_j,y_j)\in\bbR\times\sX$ for all $1\leq j\leq m$. To make the first term small, we first quantify the continuity of the kernel $k$. Let $\epsilon'=\epsilon/\norm{\beta}_1$ and define $$\eta_{\epsilon}:=\mu^{1/2}(\sX)\frac{\epsilon'}{L(k)},$$ where $L(k)$ is the Lipschitz constant of $k$. Then for every $x$ and $y$ in $\sX$ we have that $\twonorm{x-y} \leq \eta_{\epsilon}$ implies $\twonorm{k(\cdot,x) - k(\cdot,y)} \leq\mu^{1/2}(\sX)\epsilon'$.

Recall $f_\alpha  = \sum_{i=1}^n \alpha_i k(\cdot,Z_i)$. With the above result in hand we now have to make sure that at least a subset of the centers $\set{Z_i}_{i=1}^n$ of $f_{\alpha}$ are close to the centers $\set{y_j}_{j=1}^m$ of $f_\beta$ with high probability. First, define $B_j = \set{x\in\sX \mid \twonorm{x-y_j}\leq \eta_{\epsilon}, }$ and define $P_Z := P_{X,\Gamma}$. Then we obtain the following lemma:

\begin{lemma}\label{lemma:datatocenters}
Let $\epsilon>0$ and $f_\beta$, $f_\alpha$ and $B_j$, $j=1,\dots,m$ as above. Then
\begin{align*}
\prob{Z}{\forall B_j \,\, \exists Z_{i_j} \in \set{Z_i}_{i=1}^n \,\, \ni Z_{i_j} \in B_j} \geq 1 - me^{-np_0}
\end{align*}
for all $n$, where $p_0 = \min_j{\int_{B_j}f_Z (z)dz}$.
\end{lemma}

\begin{proof}[Proof of Lemma~\ref{lemma:datatocenters}]
Let the event $A_j^C = \set{\forall z \in \set{Z_i}_{i=1}^n, \,\,\, z \notin B_j}$. Then
\begin{align}
\prob{Z}{A_j^C} &= \prob{Z}{Z_1 \notin B_j,\dots,Z_n \notin B_j}\nonumber\\
& = \prod_{i=1}^n\prob{Z}{Z_i\notin B_j} \nonumber\\
&= \prod_{i=1}^n (1 - \prob{Z}{Z_i \in B_j})\nonumber\\
& = \prod_{i=1}^n(1-p_j) = (1-p_j)^n, \nonumber
\end{align}
where $p_j = \int_{B_j}f_Z (z)dz$. Let $p_0:=\min_j \set{p_j}$ and recall that for $p\leq1$ we have $1-p\leq e^{-p}$. Hence $(1-p_j)^n\leq e^{-n p_j}\leq e^{-n p_0}$, and
\begin{align*}
\prob{Z}{\cap^m A_j} & = 1 - \prob{Z}{\cup A_j^C}\\
&\geq 1 - \sum_{j=1}^m\prob{Z}{A_j^C}\\
& = 1 - \sum_{j=1}^m(1-p_j)^n \\
&\geq 1 - \sum_{j=1}^m e^{- n p_0}\\
& = 1 - m e^{- n p_0}.
\end{align*}
For this term to approach zero we need $p_0:= \min_j{\int_{B_j}f_Z (z)dz}$ to be strictly positive. This follows from the assumption that $\supp\paren{f_Z} \supseteq \sX$. Since $m$ and $p_0$ only depend on $\epsilon$ and other constants, we get
\begin{align*}
\prob{Z}{\cap^m A_j} \geq 1 - m e^{- n p_0} \rightarrow 1
\end{align*}
as $n\rightarrow\infty$.
\end{proof}

\begin{lemma}\label{lemma:infprob}
Let $\sX$, $\mu$ satisfy~\nameref{assump:xset}, $\set{(X_i,\Gamma_i)}_{i=1}^n$ satisfy~\nameref{assump:data}, and $k$ satisfy~\nameref{assump:pdkernel}. Then, $\forall\,$ $\epsilon>0$ $\,\exists\, C$ such that
\begin{align*}
\prob{X,\Gamma}{\inf_{\alpha\in A_n}\twonorm{f - f_\alpha} > \epsilon} \leq Ce^{-np_0}
\end{align*}
for sufficiently large $n$.
\end{lemma}

\begin{proof}
Let $\delta_2 = m\exp\paren{-n p_0}$. With probability $\geq 1-\delta_2$ we have that for every $j$ there is an $i_j$ such that $\twonorm{k(\cdot,z_{i_j}) - k(\cdot,y_j)} \leq \frac{\epsilon}{\norm{\beta}_1}$. Then, for $\alpha^*$ defined as
\begin{displaymath}
\alpha_i^* = \left\{
     \begin{array}{lr}
       \beta_j & : i=i_j\\
       0 & : i\neq i_j
     \end{array}
\right.
\end{displaymath}
we have
\begin{align}
\twonorm{f_{\alpha^*} - f_\beta} & = \twonorm{\mysum \alpha^*_i k(\cdot,z_i) - \sum_{j=1}^m \beta_j k(\cdot,y_j)} \nonumber\\
&\leq \sum_{j=1}^m \abs{\beta_j} \twonorm{k(\cdot,z_{i_j}) - k(\cdot,y_j)}\nonumber\\
&\leq \sum_{j=1}^m\abs{\beta_j}\mu^{1/2}(\sX)\frac{\epsilon}{\norm{\beta}_1}\nonumber\\
&=\mu^{1/2}(\sX)\epsilon.\nonumber
\end{align}

Note that if for two sequences $\set{s_n}, \set{s'_n}$ we have $s_n\leq s'_n$ for $n\geq N_0$, then $\lim_{n\rightarrow\infty}{s_n}\leq\lim_{n\rightarrow\infty}{s'_n}$, granted such limits exist. Let $s_n := \prob{X,\Gamma}{\inf_{\alpha\in A_n}\twonorm{f - f_\alpha} > 3\mu^{1/2}(\sX)\epsilon}$ and $s'_n := \delta_2$. Note that for $n$ large enough, say $n=N_0$, $R_n\geq \onenorm{\beta}$ and therefore $\alpha^*\in A_n$. So we can see that for $n\geq N_0$, the inequality
\begin{align*}
\inf_{\alpha\in A_n}\twonorm{f - f_\alpha} &\leq \twonorm{f-f_{\alpha^*}}\\
&\leq \twonorm{f-f_\sH} + \twonorm{f_\sH - f_\beta} + \twonorm{f_\beta - f_{\alpha^*}}\\
&\leq3\mu^{1/2}(\sX)\epsilon
\end{align*}
holds with probability $\geq 1-\delta_2$. That is, for $n\geq N_0$, $s_n\leq s'_n$, hence since $s'_n = \delta_2\rightarrow 0$, we get $s_n\rightarrow 0$.
\end{proof}

\begin{proof}[Proof of Theorem \ref{thm:consistency}]
\begin{align}
&\prob{X,\Gamma}{J(\alphahat)\leq4\epsilon+\paren{3\mu^{1/2}(\sX)\epsilon}^2}\\
&\geq\prob{X,\Gamma}{\set{J(\alphahat)\leq J^*+4\epsilon}\cap\set{J^*\leq\paren{3\mu^{1/2}(\sX)\epsilon}^2}}\nonumber\\
&=1-\prob{X,\Gamma}{\set{J(\alphahat)\leq J^*+4\epsilon}^C\cup\set{J^*\leq\paren{3\mu^{1/2}(\sX)\epsilon}^2}^C}\nonumber\\
&\geq 1 - \paren{\prob{X,\Gamma}{{J(\alphahat)> J^*+4\epsilon}} + \prob{X,\Gamma}{{J^*>\paren{3\mu^{1/2}(\sX)\epsilon}^2}}}.\nonumber\\
\end{align}
By Lemma~\ref{lemma:oracle} the middle term approaches zero and by Lemma~\ref{lemma:infprob} the last term does, so
\begin{align*}
\lim_{n\rightarrow\infty}\prob{X,\Gamma}{J(\alphahat)\leq\epsilon'}&=1,
\end{align*}
where $\epsilon'=4\epsilon+\paren{3\mu^{1/2}(\sX)\epsilon}^2$.
\end{proof}

\subsection{Convergence Rates of $f_\alphahat$}
The proof of Lemma~\ref{lemma:gaussrates2} is found in \cite{gnecco2008approximation}. Now let $\beta_j = \frac{c_j\onenorm{\lambda}}{m}$ and $f_\beta = \sum_{j=1}^m\beta_j k(\cdot,y_j)$, where $\lambda$, $m$, and $c_j$, $y_j$, for $j=1,\dots,m$, are as in Lemma P~\ref{lemma:gaussrates2}. Recall Lemma P~\ref{lemma:gaussrates3}:
\begin{lemma*}[Lemma~\ref{lemma:gaussrates3}]
Let $\delta_2\in(0,1)$, let $f\in\sF_k$ and let $f_\beta$ and $m$  be as above. Let $\set{(X_i,\Gamma_i)}_{i=1}^n$ satisfy assumption~\nameref{assump:data}, then with probability $\geq 1 - \delta_2$
\begin{align*}
\inf_{\alpha\in A_n}\twonorm{f_\beta - f_\alpha}\leq \epsilon_3(n,m)
\end{align*}
where $\epsilon_3(n,m): = \frac{C}{n^{1/d}}\log^{1/d}{(m/\delta_2)}$, for $C$ a constant independent of $n$ and $m$.
\hfill \qed
\end{lemma*}

\begin{proof}[Proof of Lemma~\ref{lemma:gaussrates3}]
Let $\eta = \eta(\epsilon_3): = \frac{\epsilon_3}{L\norm{\beta}_1}$, where $L$ is the Lipschitz constant of the $k$. From Lemmas~\ref{lemma:datatocenters} and~\ref{lemma:infprob} we know that with probability $\geq 1 - me^{-np_0}$ the event that for all $y_j$ there is a data point $Z_{i_j}$ such that $\twonorm{y_j - Z_{i_j}}\leq \eta$ holds. Recall $p_0 = \min_{j}\int_{B_j}f_Z(z)dz$. Now
\begin{align*}
p_0 & = \min_{j}\int_{B_j}f_Z(z)dz \\
&\geq \min_{z\in B_j}\set{f_Z(z)}\cdot \min_j\int_{B_j}dz \\
&\geq \min_{z\in B_j}\set{f_Z(z)}\cdot \min_j vol(B_j) \\
&= \min_{z\in B_j}\set{f_Z(z)}\cdot c' \eta^d,
\end{align*}
where $c'$ is the volume of the $d$-dimensional unit ball.

Now pick the $\alpha$ coefficients as in Lemma~\ref{lemma:datatocenters}. With probability $\geq 1-me^{-nc'\eta^d}$ we have:
\begin{align*}
\inf_{\alpha\in A_n}\twonorm{f_\beta - f_\alpha} &\leq \twonorm{\sum_{j=1}^m \beta_j\paren{k(\cdot,y_j) - k(\cdot,z_{i_j})}} \\
& \leq L \mu^{1/2}(\sX) \norm{\beta}_1\eta\\
& =\mu^{1/2}(\sX) \epsilon_3.
\end{align*}
Fixing $\delta_2\in(0,1)$ and setting $me^{-nc\eta^d(\epsilon_3)}=\delta_2$, where $c = c'\min_{x\in\sX}\set{f^*(x)}$, we obtain
\begin{align*}
\epsilon_3 = \frac{L\norm{\beta}_1}{(nc)^{1/d}} \log^{1/d}\paren{{m}/{\delta_2}}.
\end{align*}
Finally, noting that $\norm{\beta}_1 = \norm{\lambda}_1$, where $\lambda$ is as in Lemma~\ref{lemma:gaussrates2}, yields the desired result.
\end{proof}

We now prove Theorem~\ref{theorem:gaussrates}.

\begin{proof}[Proof of Theorem~\ref{theorem:gaussrates}]
We will use the notation and results of Lemmas~\ref{lemma:gaussrates1}, \ref{lemma:gaussrates2} and~\ref{lemma:gaussrates3}. First, note that from Lemma~\ref{lemma:gaussrates2} we have
$
\twonorm{f-f_\beta} \leq \mu^{1/2}(\sX)\norm{f-f_\beta}_\infty \leq \mu^{1/2}(\sX) \epsilon_2(m).
$
So putting the three Lemmas together we have that with probability $\geq 1 - (\delta_1 + \delta_2)$
\begin{align*}
\twonorm{f-f_{\alphahat}}^2 \leq \epsilon_1(n) + \mu(\sX)(\epsilon_2(m) + \epsilon_3(n,m))^2,
\end{align*}
or, taking into account only the dependence on $n, m, \delta_1$ and $\delta_2$ for the different $\epsilon_i$'s, we have that $\twonorm{f-f_{\alphahat}}^2$ is of the order of
\begin{align*}
\twonorm{f-f_{\alphahat}}^2 &\lesssim \log^{1/2}\paren{n/\delta_1}\frac{R_n}{n^{1/2}} + \frac{1}{m} + \log^{2/d}\paren{m/\delta_2}\frac{1}{n^{2/d}}.
\end{align*}

Now, we want the number $m$ of centers in $f_\beta$ close to but no larger than the number $n$ of data points, so we set $m=n^\theta$ for some $\theta$ such that $0<\theta<1$. Furthermore, we need $R_n$ to grow accordingly, so that $R_n = n^{c\theta}$, for $c>0$ a constant possibly dependent on $d$. This yields, ignoring the $\log$ terms for now:
\begin{align*}
\twonorm{f-f_{\alphahat}}^2 &\lesssim \frac{1}{n^{1/2 - c\theta}} + \frac{1}{n^\theta} + \frac{1}{n^{2/d}}.
\end{align*}
Setting the first two rates equal we obtain $\theta = 1/{2(1+c)}$. Note that if $d>4$ we can match the third term by setting $c=d/4-1$ to obtain an overall rate of $2/d$. Otherwise we can set $c$ to any small number to obtain a rate $1/{2(1+c)}$ slightly slower than $1/2$. Finally, for the $\log$ terms, just let $\delta_1=\delta_2=\delta/2$ and note that if $d>4$ then $\log^{1/2}$ dominates, and if $d\leq 4$ then $\log^{2/d}$ dominates.
\end{proof}

\subsection{Box rates}
First recall the definition of $\tilde{f}_\alpha$ from Section~\ref{boxrates}: For any $M$ and $\epsilon(M)$ let $f_\beta = \sum_{i=1}^{M}\beta_i k(\cdot,y_i)$ be such that $\twonorm{f - f_\beta}<\epsilon(M)$, then
$$
\tilde{f}_\alpha:=\sum_{i=1}^{M}\alpha_i k(\cdot,y_i).
$$

And, also from section~\ref{boxrates}: $\sX = [0,1]^d$, $\sF = L_{\mu}^{1}(\sX)\cap L_{\mu}^{2}(\sX)\cap \text{Lip}(\sX)$, where $\text{Lip}(\sX)$ are the Lipschitz functions on $\sX$, $\mu$ is the Lebesgue measure and $L_f$ is the Lipschitz constant of $f$. Also, $k$ is the box kernel $ k(x,y) = \frac{1}{(2\sigma)^{d}} \ind{\norm{x-y}_\infty\leq\sigma}$ defined on $\sX \pm \sigma \times \sX \pm \sigma$, and for simplicity we assume $\sigma = \frac{1}{2q}$ for $q$ a positive integer.
\begin{lemma*}[Theorem~\ref{theorem:approx_rates}]
Let $f \in \mathcal{F}$, $\supp\set{f_Z}\supset\sX$, $\set{(X_i,\Gamma_i)}_{i=1}^n$ satisfy~\nameref{assump:data} and $R_n \sim n^{(d-1)/(2d+2)}$. Let $\delta \in \paren{0,1}$, then, with probability greater than or equal to $1-\delta$
\begin{align*}
\twonorm{f-f_{\alphahat}}^2 \lesssim \frac{\log^{1/2}(n/\delta)}{n^{1/(d+1)}}
\end{align*}
\hfill \qed
\end{lemma*}

To prove this theorem we begin by reformulating Lemma~\ref{lemma:oracle}:
\begin{lemma}\label{lemma:estim_rates}
Let $\delta\in(0,1)$. Let $\set{(X_i)}_{i=1}^n$ satisfy~\nameref{assump:data} and $k$ satisfy $\sup_{x,x'\in\sX}{k(x,x')}\leq C_k$. If $\tilde{f}_{\alphahat}$ is as above, then with probability $\geq 1- \delta$
\begin{align*}
J(\alphahat)\leq \epsilon_1(n) + \inf_{\alpha\in A_n}\twonorm{f-\tilde{f}_\alpha}
\end{align*}
where $\epsilon_1(n): = \sqrt{8}C_k R_{M}\sqrt{\frac{\log{(4M/\delta_1)}}{n-1}}$.
\hfill \qed
\end{lemma}
We now take care of the term $\inf_{\alpha\in A_n}\twonorm{f-\tilde{f}_\alpha}$.
\begin{lemma}\label{lemma:simplefunction}
Let $f\in\sF$. For any $m\in\mathbb{N}$ there is a function $f_\beta$ of the form $f_\beta = \sum_{i=1}^{(mq)^d}\beta_i k(\cdot,y_i)$, where $\set{y_i}_{i=1}^{(mq)^d}\subset \sX \pm \sigma$ and $\norm{\beta}_1\leq \paren{mq}^{d-1}R_{f,\sigma}$ satisfying
\begin{align*}
\twonorm{f-f_\beta} \leq \epsilon_2(m)
\end{align*}
where $\epsilon_2(m): = \frac{L_f\sqrt{d}}{mq}$.
\hfill \qed
\end{lemma}
\begin{proof}[Proof of Lemma \ref{lemma:simplefunction}]
Let $\iota := (\iota_1,\dots,\iota_d)$ be a multi-index with positive elements and associated index $i$ related by the function $h$:
$$
i = h(\iota) = 1 + \sum_{\ell=1}^d(\iota_\ell-1)(mq)^{\ell-1}
$$
and its inverse
$$
h^{-1}(i) = \paren{ \ceil*{\frac{i \mod (mq)^1}{(mq)^0} },\ceil*{\frac{i \mod (mq)^2}{(mq)^1} } , \dots, \ceil*{\frac{i \mod (mq)^d}{(mq)^{d-1}} }  }.
$$
\begin{sloppypar}
Divide $\sX = [0,1]^d$ into $(mq)^d$ hypercube regions of equal volume to form the partition $\set{\Pi_{\ell=1}^{d} \brac{\frac{\iota_\ell-1}{mq},\frac{\iota_\ell}{mq}}}_{\iota_1,\dots,\iota_d =1}^{(mq)} = \set{T_i}_{i=1}^{(mq)^d}$, where $T_i = \Pi_{\ell=1}^{d} \brac{\frac{\iota_\ell-1}{mq},\frac{\iota_\ell}{mq}}$. Now, let
$$
f_m = \sum_{i=1}^{(mq)^d}f(\xbar_i)\ind{T_i}
$$
where $\xbar_i \in T_i$. Any choice of $\xbar_i$ works but for clarity we choose $\xbar_i = (\iota_1/mq,\cdots,\iota_d/mq)$. Note that $f_m$ is close to $f$:
\begin{align*}
\twonorm{f-f_m}^2 &\leq \int_{\sX}(f(x) - f_m(x))^2dx \\
& = \sum_{i=1}^{(mq)^d}\int_{T_i}\paren{f(x) - f_m(x)}^2dx \\
& = \sum_{i=1}^{(mq)^d}\int_{T_i}\paren{f(x) - f(\xbar_i)}^2dx \\
& \leq \sum_{i=1}^{(mq)^d}\int_{T_i}\paren{L_f\twonorm{x - \xbar_i}}^2dx \\
& \leq \sum_{i=1}^{(mq)^d}\int_{T_i}\paren{L_f\frac{\sqrt{d}}{mq}}^2dx \\
& \leq \paren{\frac{L_f\sqrt{d}}{mq}}^2.
\end{align*}
Hence
$$
\twonorm{f-f_m} \leq \frac{L_f\sqrt{d}}{mq}.
$$
\end{sloppypar}
\begin{sloppypar}
Now we note that $f_m$ can also be expressed as a sum of fixed bandwidth kernels:
\begin{align*}
f_m &= \sum_{i=1}^{(mq)^d}\beta_i (2\sigma)^d k(\cdot,y_i),
\end{align*}
where
$$
y_i = \brac{\frac{\iota_1-1}{mq}+\sigma,\cdots,\frac{\iota_d-1}{mq}+\sigma}^T
$$
and $\beta$ is as follows.
Let $\beta_1 = f(\xbar_1)$ and
$$
\beta_i = f(\xbar_i) - \sum_{\kappa = s_i}^{S_i}\beta_{h(\kappa)} - f(0)\ind{\iota_\ell =1 \,\, \forall \ell}
$$
for $i\leq2\leq(mq)^d$, where $s_i = (\max\set{1,\iota_1-(m-1)},\cdots,\max\set{1,\iota_d-(m-1)})$ and $S_i = (\max\set{1,\iota_1-1},\cdots,\max{1,\iota_d-1})$ are multi-indices. Note that the $\beta_i$'s sequentially capture the residual of the function $f_m$ as we travel along the $T_i$ regions.

To find $\norm{\beta}_1$ note that since $\twonorm{\xbar_1 - 0}\leq L_f\sqrt{d}/(mq)$, we have $\abs{\beta_1} \leq f(0) + L_f\sqrt{d}/(mq)$. Also note $\beta_i = f(\xbar_i) - f(\xbar_{i-1})$ for $2\leq \iota_1\leq m$, hence for $2\leq i\leq m$
\begin{align*}
\abs{\beta_i} &= \abs{f(\xbar_i) - f(\xbar_{i-1})} \\
&\leq L_f\twonorm{\xbar_i - \xbar_{i-1}} \\
&\leq L_f\frac{\sqrt{d}}{mq}
\end{align*}
\end{sloppypar}

For larger $\iota_1$ we have $\beta_i = f(\xbar_i) - f(\xbar_{i-1}) +\beta_{i-m}-f(0)\ind{\iota_\ell =1 \,\, \forall \ell}$.
 Note that when $\iota_1 = m+1$ we have lost influence of $\beta_1$, so
$$
\abs{\beta_{m+1}}\leq \abs{f(\xbar_i) - f(\xbar_{i-1})} + \abs{\beta_1} + f(0) \leq 2L_f\frac{\sqrt{d}}{mq}+2f(0)
$$
and, similarly, $\abs{\beta_i}\leq 2L_f\frac{\sqrt{d}}{mq}$ for $m+2\leq i \leq2m$. The process continues such that, in general
\begin{align*}
\abs{\beta_i}\leq \ceil*{\frac{\iota_1}{m}}\paren{ L_f\frac{\sqrt{d}}{mq} + f(0)\ind{\iota_\ell =1 \,\, \forall \ell}}
\end{align*}
for $i\leq mq$. Adding these we obtain $\frac{q+1}{2}(L_f\sqrt{d} + qf_0)$. Denote this quantity by $e'$. Then this process is repeated over every dimension, having $q$ chunks of multiples of $e'$, the first multiple is $1m$, the second $2m$, and so on. The final sum is then
$$
\norm{\beta}_1\leq \paren{mq}^{d-1}\paren{\frac{q+1}{2}}^2\paren{qf(0) + L_f\sqrt{d}}.
$$
Therefore
$$
(2\sigma)^d\norm{\beta}_1 = \frac{\paren{m}^{d-1}}{q}\paren{\frac{q+1}{2}}^2\paren{qf(0) + L_f\sqrt{d}}.
$$
\end{proof}


\begin{proof}[Proof of Theorem \ref{theorem:approx_rates}]
This proof is similar to the proof of Theorem~\ref{theorem:gaussrates}. Combining Lemmas~\ref{lemma:estim_rates}, \ref{lemma:simplefunction}, setting $R_{n'}\sim m^{(d-1)}$, ignoring the $log$ terms for now and considering only the dependence on $n$ and $m$ we obtain
\begin{align*}
\twonorm{f-f_{\alphahat}}^2 &\lesssim \frac{m^{d-1}}{{n}^{1/2}} + \frac{1}{m^2}.
\end{align*}
Setting these terms equal we obtain $m = n^{1/(2d+2)}$, with an overall rate of $n^{-1/(d+1)}$. Adding the log term we obtain
$$
\twonorm{f-f_{\alphahat}}^2 \lesssim \frac{\log^{1/2}(n^{d/(2d+2)}/\delta)}{n^{1/(d+1)}} \lesssim \frac{\log^{1/2}(n/\delta)}{n^{1/(d+1)}}.
$$
\end{proof}

\bibliographystyle{IEEEtran}
\bibliography{efrenbib}

\end{document}